\documentclass[letterpaper]{article}
\pdfoutput=1 
\usepackage{aaai}
\usepackage{times}
\usepackage{helvet}
\usepackage{courier}
\usepackage{graphicx} 
\usepackage{subfigure}

\usepackage{algorithm}
\usepackage{algorithmic}

\usepackage{amssymb}
\usepackage{amsthm}

\newcommand{\TODO}[1]{\marginpar{Robby TODO}}

\newcommand{\w}{\vec{w}}
\newcommand{\D}{{\vec \Delta}}
\newcommand{\T}{^\intercal}

\newtheorem{trm}{Theorem}
\newtheorem{assumption}{Assumption}

\begin{document}
%
\title{Coactive Learning for Locally Optimal Problem Solving}

\author{Robby Goetschalckx, Alan Fern, Prasad Tadepalli\\
School of EECS,\\
Oregon State University, Corvallis, Oregon, USA
}
\maketitle
\begin{abstract}
\begin{quote}

Coactive learning is an online problem solving setting where the solutions
provided by a solver are interactively improved by a domain expert, which
in turn drives learning.
In this paper we extend the
study of coactive learning to problems where obtaining a
globally optimal or near-optimal solution may be intractable or where an expert
can only be expected to make small, local improvements to a candidate solution.
The goal of learning in this new setting is to minimize the cost
as measured by the expert effort
over time. We first establish theoretical bounds
on the average cost of the existing coactive Perceptron
algorithm. In addition, we consider new online algorithms that use
cost-sensitive and Passive-Aggressive (PA) updates, showing similar
or improved theoretical bounds. We provide an empirical evaluation
of the learners in various domains, which show that the Perceptron based
algorithms are quite effective and that unlike the case for online
classification, the PA algorithms do not yield significant performance
gains.

\end{quote}
\end{abstract}

\section{Introduction}

This work is motivated by situations where a domain expert must solve a sequence
of related problems. One way to reduce the expert's effort is for an
automated solver to produce initial solutions that can then be improved, if
necessary, by the expert with less effort than solving the problem from scratch.
This requires that the solver has a good estimate of the expert's utility
function, which is often unknown and must be learned through experience. This
general notion of online learning from the improved solutions of an {\em in
situ} expert is captured by the framework of \emph{coactive learning} \cite{DBLP:journals/corr/abs-1205-4213}, \cite{icml2013_raman13}, \cite{raman12}, \cite{ecml:raman13}.
%

The current state-of-the-art in coactive learning generally assumes that the
solver can either find a globally optimal solution (according to its current estimate
of the utility function) to the problem, or at least a solution which can be proven to be at least $\alpha$-close to the optimal solution (for example, using a greedy algorithm for a sub-modular utility function \cite{raman12}, \cite{ecml:raman13}).
It is also assumed that the expert can always improve the
solution if there is a possible improvement.

Unfortunately in many planning and combinatorial optimization problems, neither
of these assumptions is realistic. The current paper relaxes these assumptions
by extending coactive learning to the cases where the system can only guarantee
locally optimal solutions and the expert is only assumed to make a sequence of
local improvements, but not required to return a locally optimal solution.


For example, consider a Traveling Salesperson Problem, where each edge of the
graph is described by a number of numerical features, and the edge costs are
known to be a linear combination of these features with unknown weights. We have
an approximate solver which, for a given estimate of the edge-cost function,
will give a reasonably good (but perhaps not optimal) solution to the TSP
problem. The solver presents a candidate solution to an expert (who knows the
actual cost function). The expert can either accept the solution or spend some
effort making improvements to it. The more accurate the learner's estimate of the cost
function, the less effort the expert needs to spend to adapt the
candidate solution into one which is acceptable.


More generally, we assume that we have a planning or a combinatorial
optimization problem, where finding a globally optimal solution for the problem
is intractable, and where it is hard for an expert to do anything more than a
sequence of local improvements to a candidate solution. We assume that we have a
black-box solver which, for a given problem and current estimate of the cost
function, will output a locally optimal solution.
Assuming that the expert and the system have the same notion of `locality', changes made by the expert
will be very likely due to the estimate of
the distance function being different from the actual distance function. With
these assumptions, the expert feedback will allow the learner to improve its
estimate of the proper cost function, eventually removing the need for
the expert.

It is important to note that the local optimality assumptions of our framework
strictly generalize those of prior coactive learning work, where the learner and the
expert were assumed to be within a factor of $\alpha$ of the true optimal solution.
 Another important aspect our framework is that we allow the learner to observe the
effort expended by the expert, or cost, when producing improvements.
The learner can exploit this additional feedback to put more emphasis on
mistakes that are more costly from the expert's perspective, ideally leading to
a faster reduction of expert effort.

Perhaps the closest related work to ours is on learning trajectory
preferences for robot manipulation \cite{jainlearning}. In this work, the authors
describe a coactive learning approach that learns a scoring function based on
improvements made by experts to automatically generated trajectories
of robot manipulation tasks. 
Our main contributions in
relation to this work are to formalize coactive learning in the context
of locally optimal planning and demonstrate better learning algorithms
with tighter performance bounds by making them cost-sensitive.



In addition to introducing the generalized coactive learning framework, our
other primary contribution is to propose and analyze the cumulative cost of
simple online learning algorithms for linearly parameterized cost functions. In
particular, we describe cost-insensitive and cost-sensitive variants of the
Perceptron and Passive Aggressive (PA) learning algorithms.
We empirically evaluate these algorithms on four different domains and show that, the cost-sensitive Perceptron tends to perform better than all others when
there is noise in the data. Also,
unlike in classification setting, here the perceptron algorithms perform
better than the PA algorithms.
One reason for this is that the PA algorithms
aggressively correct for any new data, meaning that any amount of noise can
degrade the performance significantly.

\section{Problem Setup}
\label{sec:setup}

Let $\mathcal{X}$ be the set of all problems in the domain of interest, e.g. a
class of TSP graphs, and $\mathcal{Y}$ be the set of all possible problem
solutions, e.g. possible tours.
The quality of solutions is judged by the expert's \emph{utility} function
$U(\langle x,y \rangle) \to \mathbb{R}$ which gives the relative merit of a
given candidate solution for a given problem.

For the purposes of learning we assume that the utility function can be well
approximated by a linear function $\hat{U}(\langle x,y \rangle) =
\w\T\vec{\phi}(\langle x,y \rangle )$. Here $\vec{\phi}(\langle x,y \rangle )$
is a real-valued feature function of problem-solution pairs and $\w$ is a
real-valued weight vector. In this work, we assume a bounded feature vector
length: $\| \vec{\phi}(\cdot) \| \leq R$.

A critical part of our coactive learning framework is the improvement of
solutions by the expert. For this purpose we consider a set $\mathcal{O}$ of
operators which can be used by the expert to transform a candidate solution into
a different candidate solution: $O_i \in \mathcal{O}: \langle x,y \rangle  \to
\langle x,y' \rangle$. The specific operators are highly domain and
user-interface dependent. For example, in a TSP application, the operators
might allow for local modifications to a current tour.

We assume a black-box solver $S:
\mathcal{X} \to \mathcal{Y}$ which can be parameterized by the utility function
weight vector $\w$ in order to optimize solutions according to $\hat{U}$. For
many planning problems, including TSP, exactly optimizing $\hat{U}$ will be
intractable. Thus, for the purposes of our analysis we make the more realistic
assumption that $S$ is locally optimal with respect to $\hat{U}$ and the set of
expert operators $\mathcal{O}$.
\begin{assumption}[Locally Optimal Solver]
\label{assump:opt}
$\nexists O_i \in \mathcal{O}:
\hat{U}(O_i(\langle x_t,S(x_t)\rangle)) > \hat{U}(\langle x_t,S(x_t)\rangle)$
\end{assumption}

We can now specify our coactive learning protocol. The learner is presented with
an arbitrary sequence of problems $x_0, x_1, x_2, \ldots$ from $\mathcal{X}$.
Upon receiving each $x_t$, the learner uses $S$ to present the expert with a
candidate solution $y_t$ based on its current utility estimate $\hat{U}$. The
expert will then perform a sequence of operators $O^{(1)} \ldots O^{(l)}$ to the
solution as long as significant local improvements are possible according to the
true utility function $U$. We formalize this as a second assumption:
\begin{assumption}[Local Expert Improvement]
\label{assump:exp}
If we denote, $z^{(0)} = \langle x_t,y_t\rangle$, $z^{(i)} = O^{(i)}(z^{(i-1)})$, there is a
constant $\kappa > 0$ such that $U(z^{(i+1)}) \geq U(z^{(i)} ) + \kappa$.
\end{assumption}
The constant $\kappa$ reflects the minimum utility improvement required for the
expert to go through the effort of applying an operator. Thus, according to this assumption
the expert will monotonically improve a candidate solution
until they are satisfied with it or they have reached diminishing returns for local
improvements. The expert cost $C_t$ for
example $x_t$ is then equal to the number of operator applications, which
reflects the amount of expert effort spent improving the candidate solution.
The average cumulative cost $\frac{1}{T}\sum_{t=0}^T C_t$ quantifies the average
effort of the expert over $T$ steps. The goal of learning is to quickly diminish this
cost.

We note that it is straightforward to generalize the framework to account for
operator dependent costs.

\section{Learning Algorithms}

We consider online Perceptron-style algorithms that maintain a weight vector $w_t$ that is initialized to $w_0=0$ and potentially updated after each training experience. At iteration $t$ the learner presents the expert with solution $S(x_t)$ and observes the corresponding improvement operator sequence with cost $C_t$.
Whenever $C_t > 0$ our learning algorithms will adjust the weight vector in the direction of $\D_t = \vec{\phi}(x_t,y') - \vec{\phi}(x_t,S(x_t))$.

\begin{algorithm}[tb]
\label{alg:update}
\caption{CoactiveUpdate(problem $x_t$, black-box solution $y_t$, expert solution $y'$ (see text), cost $C_t$)}
\begin{algorithmic}
\IF{$C_t > 0$}
\STATE $\D_t := \vec{\phi}(x_t,y') - \vec{\phi}(x_t,y_t)$
\STATE $\w_{t+1} = \w_t + \lambda_t\D_t$
\ENDIF
\end{algorithmic}
\end{algorithm}

Algorithm 1 
gives the schema used by our learning algorithms for updating the weight vector after the interaction involving problem $x_t$. In the tradition of the Perceptron, this update is passive in the sense that whenever $C_t = 0$, meaning that the expert is satisfied with the solution produced, the weight vector is not altered. Our four algorithms differ only in the choices they make for the learning rate $\lambda_t$, which controls the magnitude of the weight update. Below we specify each of our four variants that are analyzed and empirically evaluated in the following sections.

{\bf Perceptron (PER) ($\lambda_t = 1$).} This instance of the algorithm corresponds to the classic Perceptron algorithm,
which uses a uniform, cost-insensitive learning rate. This so-called
\emph{preference perceptron} algorithm was previously proposed for coactive learning, although its analysis was for their more restrictive assumption of global optimality  \cite{DBLP:journals/corr/abs-1205-4213}.

{\bf Cost-Sensitive Perceptron (CSPER) ($\lambda_t = C_t$)}.
Notice that the above Perceptron algorithm completely ignores the observed cost $C_t$. Intuitively, one would like to put more emphasis on examples where the cost is high, to potentially decrease the cumulative cost more rapidly. The Cost-Sensitive Perceptron (an online version of the perceptron introduced in \cite{geibel2003perceptron}) takes the cost into account by simply allowing the learning rate to scale with the cost, in this case being equal to the cost. This is an approximation of what would happen if the original Perceptron update was applied $C_t$ times per instance.
This is similar to a standard approach to obtain an example-dependent cost-sensitive learner \cite{zadrozny2003cost}, where the probability of each example to be considered by the algorithm is proportional to the cost of misclassifying it. The algorithm presented here can be seen as a simplified online version of this.

{\bf Passive Aggressive (PA) ($\lambda_t = \frac{M -  \w_t \T \D_t}{\| \D_t \|^2}$).} The framework of PA algorithms has been well-studied in the context of more traditional online learning problems \cite{crammer2006online} and has shown advantages compared to more traditional Perceptron updates. Here we consider whether the PA framework can show similar advantage in the context of coactive learning.

The PA update principle is to be passive, just as for the Perceptron algorithms. In addition, unlike the Perceptron, PA updates are aggressive in the sense that the weights are updated in a way that are guaranteed to correct the most recent mistake. In our coactive setting, this PA principle corresponds to minimizing the change to the weight vector while ensuring that the new vector assigns the improved solution $y'$ a higher utility than $y_t$. More formally:
$$w_{t+1} = \arg\min_{\w} \|\w - \w_t\|^2,\;\;s.t.\;\; \w_{t+1} \T \D_t = M$$
where $M$ is a user specified target value for the margin. As in prior work \cite{crammer2006online} the solution to this problem is easily derived via Lagrangian methods resulting in an update with $\lambda_t = \frac{M -  \w_t \T \D_t}{\| \D_t \|^2}$. Note that either version of the Perceptron algorithm may either not correct the mistake, or update the weights much more than necessary to correct the mistake. In our experiments we set $M=1$ throughout, noting that in the coactive learning setting the specific value of $M$ only results in a different scaling of the weight vectors, without changing the behavior or the obtained costs.
%


{\bf Cost-Sensitive PA (CSPA) ($\lambda_t = \frac{C_t -  \w_t \T \D_t}{\| \D_t \|^2}$).} Since the margin $M$ indicates how strongly an example influences the weight vector, a natural way to include the cost into the PA algorithm is by making the margin for an update equal to the cost. This approach effectively instantiates the idea of margin-scaling, e.g., as employed for structured prediction \cite{Taskar2005learning}, within the PA framework.
%

\section{Cost Bounds}
\label{sect:theory}

In this section we will present upper bounds on the average cost over time for all four algorithms presented in the previous section, in the realizable learning setting where the actual utility function is a linear function of the feature vector, represented by the weight vector $\w^*$. Since the cost is an integer, this effectively provides a bound on the number of examples where the expert is able to make any improvement on the solution.

The most similar prior analysis done for coactive learning focused on the Perceptron algorithm under the assumptions of global optimality. That analysis provided a bound on the average regret, rather than expert cost, where regret is the utility difference between the globally optimal solution and the learner's solution. Thus the novelty of our analysis is in directly bounding the expert cost, which is arguably more relevant, generalizing to the setting of local optimality, and analyzing a wider set of algorithms.
The analysis in \cite{DBLP:journals/corr/abs-1205-4213} also provides a lower bound on the average difference in \emph{utlity} between the expert solution and initial solution of $\Omega(\frac{1}{\sqrt{T}})$. For one of our algorithms, the upper bound on the \emph{cost} we obtain is better than this, ($O(\frac{1}{T}))$, indicating the difference between utility and cost. In fact, working out the specific example given as proof of the bound in \cite{DBLP:journals/corr/abs-1205-4213} does provide a worst-case example, where the average cost is equal to the upper bound.

We will make one extra assumption for our theoretical analysis in addition to Assumptions 1 and 2 regarding local optimality.
\begin{assumption}
\label{assum:localopt}
If $C_t > 0$, then $\hat{U}(\langle x,y_t' \rangle) < \hat{U}(\langle x,y_t\rangle)$. 
\end{assumption}
\noindent This assumption assures us that the learner will favor its own solution $y_t$ over the final locally optimal improvement provided by the expert. 
 In practice, this can be guaranteed by observing the entire \emph{sequence} of candidate solutions the expert goes through to obtain their final solution and taking the last candidate in the sequence where the learner disagrees with it being an improvement on the original candidate solution. There will be at least one such solution in the sequence because of the assumption of local optimality of the original candidate solution. 
Of course, this implies that the actual effort will potentially be larger than the effort until the presented solution was found. However, since any such case has a cost of at least 1, we still have an upper bound on  the number of times where the expert is able to provide improvements on the solution by local changes.

We now state the main result providing cost bounds for our four algorithms. The proof can be found in appendix A.
\begin{trm}
\label{theorem}
Given a linear expert utility function with weights $\w^*$, under Assumptions 1-3 we get the following average cost bounds:
\begin{tiny}
\begin{eqnarray*}
\mbox{PER:} & \frac{1}{T} \sum_{t=0}^{T} C_t &\leq \frac{2R||\w^*||}{\kappa \sqrt{T}} \\
\mbox{CSPER:} & \frac{1}{T} \sum_{t=0}^{T} C_t &\leq \frac{1}{T} \sum_{t=0}^{T} C_t^2 \leq \frac{4R^2||\w^*||^2}{\kappa^2 T} \\
\mbox{PA:} & \frac{1}{T} \sum_{t=0}^{T} C_t  &\leq \frac{4R^2||\w^*||^2}{\kappa^2 \sqrt{T}} \\
\mbox{CSPA:} & \frac{1}{T} \sum_{t=0}^{T} C_t &\leq \frac{1}{T} \sum_{t=0}^{T} C_t ^2\leq \frac{4R^2||\w^*||^2}{\kappa^2 \sqrt{T}}
\end{eqnarray*}
\end{tiny}
\end{trm}

We first observe that the cost-sensitive Perceptron bound decreases significantly faster than the other bounds, achieving a rate of $O(1/T)$ compared to $O(1/\sqrt{T})$. This implies that there is a constant bound on the summed costs. At least when comparing PER and CSPER, this shows the potential benefit of introducing cost into the learning process. Rather, we do not see such a rate reduction for cost-sensitive PA versus PA, which both achieve the same bound. It is not clear whether or not this is due to a loose analysis or is inherent in the algorithms. We do not yet have lower bounds on the rates of these algorithms, which would shed light on this issue. We do see however, that the bounds for both cost-sensitive algorithms hold for the average squared cost, which are strictly better bounds in our setting. Thus, even in the PA case, these results suggest that there is a potential advantage to introducing cost into learning.

Similar error bounds for the PER and CSPER algorithms in the case where there is noise in the expert's assessment, i.e. the expert might mis-estimate the relative utilities of the original solution and a suggested improvement by a term $\xi_t$ are presented in appendix B.


\section{Experiments}
Experiments were performed in 4 different problem domains. All reported results are averages over 10 runs of the experiments.
In the first 3 domains the utility function is a linear function of the features. 
In each domain, two different settings were tested. In one of the settings, there are 10 dimensions to the feature vectors. In the other, there are 11 features, but one of the features is only available to the expert, not to the solver. This adds noise to the experiment, and learning a weight vector which does not incur any cost is impossible.  The coefficients of the expert weight vector were generated uniformly at random from $[0,1]$.

In the cases where even the optimal weight vector does not guarantee non-zero cost, this average cost will be shown in the results, labeled `expert'.
All graphs shown have a log-scale on the $Y$-axis, unless reported otherwise.

\def \imageWidth {0.66}

\begin{figure*}[h!tb]
\centering
\subfigure[Path Planning, No noise]{
\includegraphics[width=\imageWidth\columnwidth]{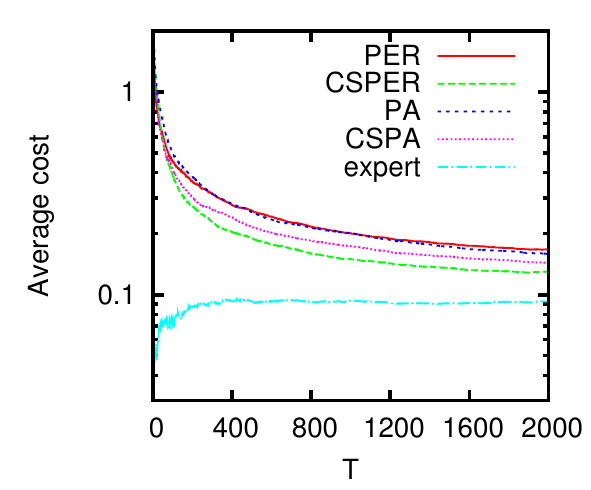}
\label{fig:cube_0_1}
}
\centering
\subfigure[Path Planning, Noise]{
\includegraphics[width=\imageWidth\columnwidth]{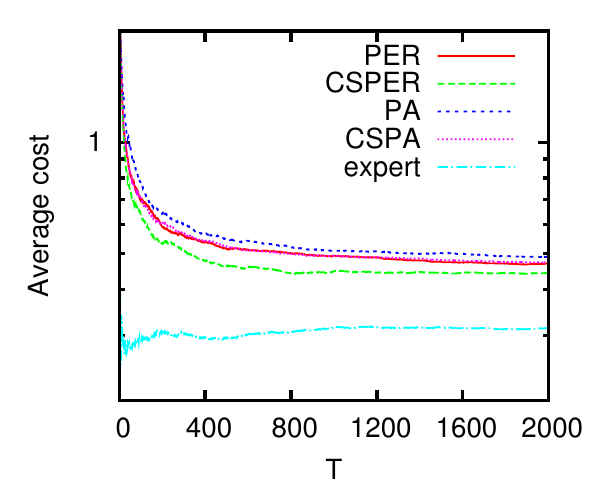}
\label{fig:cube_1_1}
}
\centering
\subfigure[TSP, No noise]{
\includegraphics[width=\imageWidth\columnwidth]{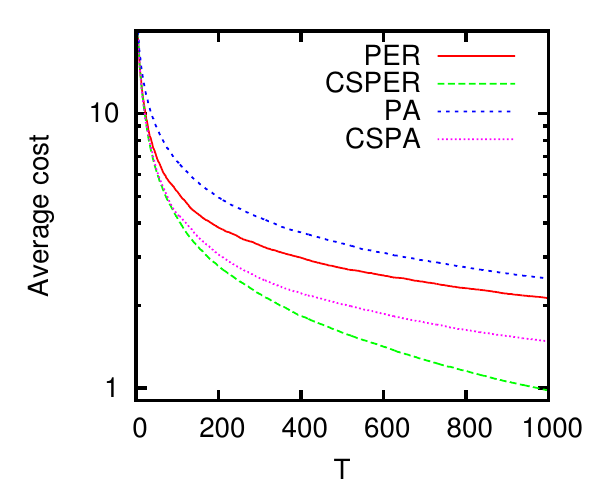}
\label{fig:TSP_0_1}
}
\centering
\subfigure[TSP, Noise]{
\includegraphics[width=\imageWidth\columnwidth]{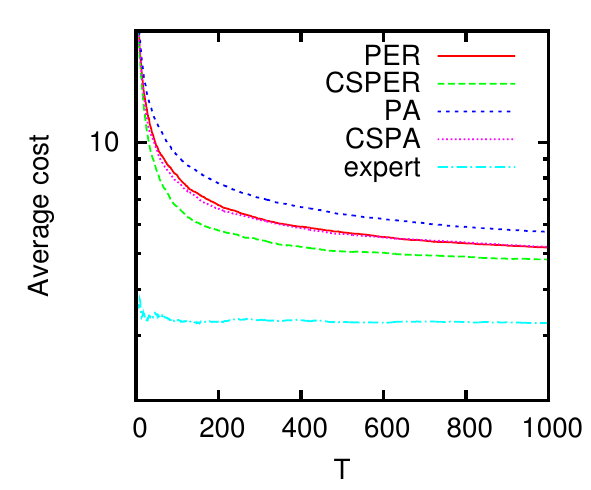}
\label{fig:TSP_1_1}
}
\centering
\subfigure[multi-TSP, No noise]{
\includegraphics[width=\imageWidth\columnwidth]{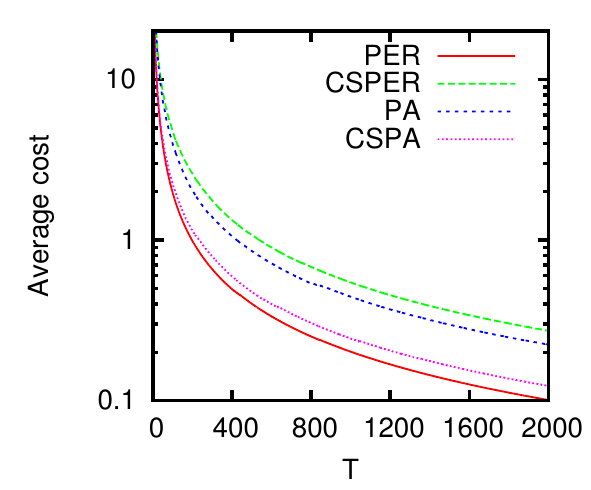}
\label{fig:mTSP_0_1}
}
\centering
\subfigure[multi-TSP, Noise]{
\includegraphics[width=\imageWidth\columnwidth]{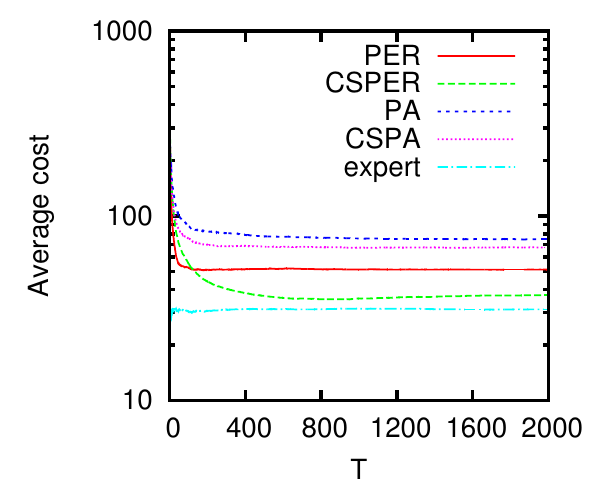}
\label{fig:mTSP_1_1}
}
\centering
\subfigure[Learning to Rank, Unlimited budget]{
\includegraphics[width=\imageWidth\columnwidth]{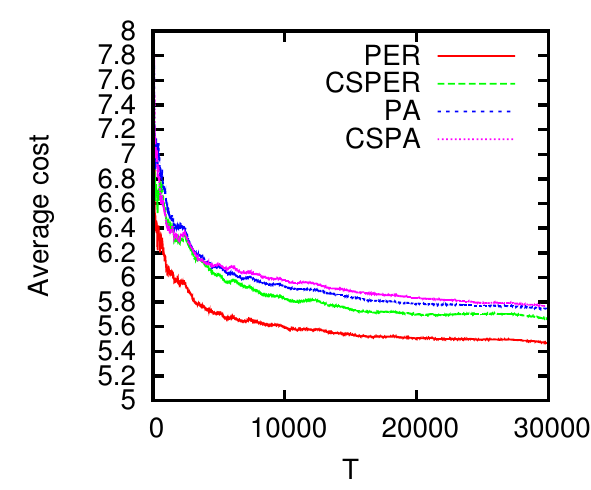}
\label{fig:webscope_raw}
}
\centering
\subfigure[Learning to Rank, Limited budget]{
\includegraphics[width=\imageWidth\columnwidth]{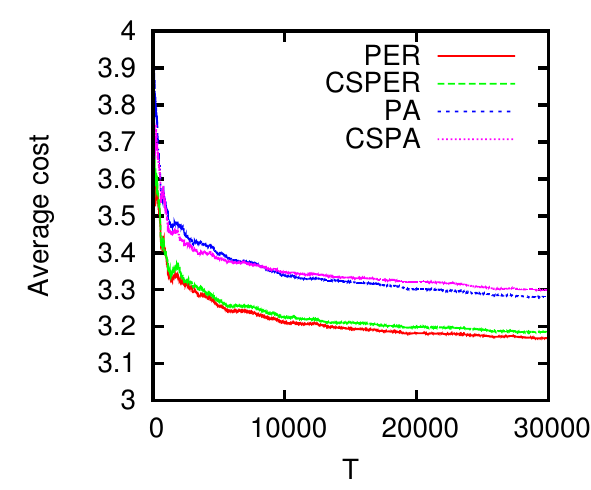}
\label{fig:webscope_lim}
}
\label{fig:LTR}
\caption{Average cost for the experiments}
\end{figure*}

{\bf Path Planning}
The first task is a simple path planning domain. The environment consists of a 7-dimensional hypercube, where the solver needs to find the optimal path of length 7 from one corner to the diagonally opposite corner. There are 7! such possible paths. Each edge is described by a feature vector, and the feature vector of a total path is the sum of the vectors describing the edges used in the path.

The solver uses a simple 2 step lookahead search. The expert can improve this trajectory by looking at three subsequent moves and reordering them, if any such reordering gives an improvement of at least $\kappa = 0.1$, until no such improvement is possible. Both a noisy and noise-free setting were tested. Note that the expert uses a different set of optimizations from the solver. This means that even with  perfect weight vector, a cost of more than 0 is unavoidable.
Results here (as shown in figures~\ref{fig:cube_0_1} and \ref{fig:cube_1_1}) show that CSPER performs best, and PA performs worst.


%


{\bf TSP}
To show the feasibility of the approach for a more interesting problem, a Traveling salesperson problem was used. A random set of 20 points was generated, with their connecting edges described by 100 dimensional feature vectors (with or without an extra 10 hidden dimensions). The solver uses a fast approximate solver. It keeps a set of unvisited points, and starts with a path containing only 1 point. The solver determines which unvisited point can be added to the current path at the least extra cost, and adds the point, until all points are visited. It then performs an optimization by performing 2-opt updates. Each such an update removes two edges from the path and reconnects the path in the shortest possible way.

The expert also uses 2-opt updates to improve the solution, but only performs an update if it decreases the path cost by at least  $\kappa$. The cost is the number of such updates performed until a local optimum is reached.
Features were generated according to a uniform $[0,1]$ distribution. The slack variable $\kappa$ has value $0.1$.
Results are shown in figures~\ref{fig:TSP_0_1} and \ref{fig:TSP_1_1}. As in the previous experiments, CSPER outperforms the other algorithms.


{\bf Multi-TSP}
Multi-TSP (mTSP) is a generalization of TSP to multiple salespersons. Given a set of 40 points, and 4 salespersons each with their own start and end locations among the given points, the goal is to determine a tour for each of the salespersons such that the total tour cost is minimized, and each point is visited exactly once by at least one salesperson.

For this experiment, a random set of 40 points are generated with their connecting edges described by feature vectors. The solver is a generalization of the TSP solver introduced in the previous section. It finds the tours for each salesperson incrementally, by finding a point that can be added to the existing tours with least cost. It repeats this strategy until no more points are left. It then applies 2-OPT adaptations to find a locally optimal solution. Similar to the TSP experiment, the expert uses 2-OPT heuristic to improve tours.

Results are shown in figures~\ref{fig:mTSP_0_1} and \ref{fig:mTSP_1_1}. For the noisy setting, CSPER clearly outperforms the other algorithms. For the settings without noise, it performs worst, however. One possible explanation is that it focuses too much on the first few data points (where costs are high), relative to later points. The regular perceptron does not have this behavior, and results for this algorithm are good in all the settings.





{\bf Learning to Rank Dataset}
A final experiment was performed using the Yahoo! Learning to Rank dataset \cite{LTRC}, a real-world dataset consisting of web-based queries and lists of documents which were returned for those queries. Human experts labeled the documents with scores from 0 to 4 according to relevance to the query. All documents are described by a 700-dimensional feature vector. The learner sorts all documents according to decreasing estimated utility. The expert then iterates over this list, exchanging the order of two subsequent elements if the latter one is  substantially better than the first (a score difference of 2 or more). The effort is the number of such exchanges until no such exchanges are possible. All reported results are averages over 10 random permutations of the original dataset.

Results are shown in figure~\ref{fig:webscope_raw} (notice the linear scale on the Y-axis). The classic perceptron algorithm performs best, with CSPER second-best. The relatively poor performance of the PA-based algorithms can be easily explained by the noise in the dataset. Since the PA algorithms aggressively correct the last observed error, any amount of noise could have a potentially disastrous effect.
Note that CSPER does not perform as well as in most of the other experiments. One possible reason for this is that, as in the noise-free multi-TSP setting, the effort for the first few examples was large, resulting in a large jump in the weight vector. Subsequent costs are much smaller, which means it takes longer to fine-tune the weight vector.
To test this, the experiment was repeated using an upper bound on the effort the expert was willing to make. For each example, a random integer between 5 and 15 was generated as the maximum ``budget'' for the expert. Results are shown in Figure~\ref{fig:webscope_lim} (notice the linear scale on the Y-axis). It is clear that this bound on the maximum cost greatly benefits the CSPER algorithm, which now only performs slightly worse than the perceptron.

%
%
%

\section{Conclusions}

In this paper, a novel problem statement for coactive learning was introduced. The main difference with existing work on coactive learning is that the problems are assumed to be so hard that finding the globally optimal solution is intractable, and an expert can be expected to only make small, local changes to a solution. We assume that there exist
locally optimal solvers that can provide high-quality solutions when having an accurate estimate of the expert's utility function.

Four algorithms were presented for this task. Since the objective is to minimize the effort spent to improve the candidate solutions, two of the algorithms directly take this cost into account in their update functions.
Theoretical bounds on the average cost for the four algorithms were shown, where the cost-sensitive perceptron algorithm was shown to have a much stronger bound. Empirically it was verified that in most settings, the cost-sensitive versions of the algorithms outperform their cost-insensitive versions. The cost-sensitive perceptron performs best in most datasets, specifically when noise was present. However, in some cases it was observed that cost sensitivity hurt performance in cases where early high-cost problems significantly alter the weight vector. Our empirical results suggest that bounding the maximum change in weight vector can help in these cases. This leaves an open question of how to most robustly and effectively take cost into account in the co-active learning setting. 

Our final empirical observation was that the Passive Aggressive algorithms perform the worst, which is in contrast to such algorithms in the context of online classification. The reason for this might be that the algorithms are too aggressive in their updates,  effectively overfitting the last seen example.
This leads to some possible directions for future work including investigation of (mini-)batch learning algorithms and versions of the PA algorithm which limit the aggressiveness \cite{crammer2006online}.

\section*{Appendix A: Proof of Theorem 1}

\begin{proof}
Each bound is derived using a similar strategy in the tradition of Perceptron-style analysis. The proof ignores all steps where no update occurs since these have cost equal to 0 and would only reduce the average cost. First, we show an upper bound on $||\w_{T+1}||^2$. Then, for each algorithm, we provide a lower bound on $\w_{T+1}\T\w^*$ in terms of the sum of costs $\sum_{t=0}^{T}C_t$ or sum of squared costs  $\sum_{t=0}^{T}C_t^2$.

Using the general update  $\w_{T+1} = \w_T + \lambda_T\D_T$, we get the following bound for $||w_{T+1}||^2$:
$||\w_{T+1}||^2 = ||\w_T||^2 + 2 \lambda_T \w_T\T \D_T + \lambda_T^2 \D_T\T \D_T$
$\leq  ||\w_T||^2 +  \lambda_T^2\D_T\T \D_T$
$\leq  ||\w_T||^2 + \lambda_T^2(2R)^2 \leq 4R^2 \sum_{t=0}^{T} \lambda_T^2$
where the last step follows by induction. The second step follows from Assumption~\ref{assum:localopt}.

Then, we use the Cauchy-Schwarz inequality $\w_{T+1}\T \w^*   \leq ||\w_{T+1}|| \cdot ||\w^*||$ to obtain an upper bound on the sum of costs, which can be trivially turned into a bound on the average cost. Note that in the setting of the paper, where the costs are natural numbers, the sum of the squares of the costs is an upper bound on the sum of the costs:
$\sum_{t=0}^{T}C_t \leq \sum_{t=0}^{T}C_t^2$
Below we derive bounds for each algorithm in turn.

{\bf Perceptron.} We get the upper bound $||\w_{T+1}||^2 < 4R^2T$.
For the lower bound on $\w_{T+1}\T\w^*$ we get:
$
\w_{T+1}\T\w^* = \w_T\T\w^* + \D_T \T \w^* \geq  \w_T\T\w^* + \kappa C_T \geq \kappa \sum_{t=0}^{T} C_t
$
Here the first inequality follows from Assumptions 2 and 3. In particular by Assumption 2 we have that for any step where $C_t > 0$, $U(x_t,y')-U(x_t,y_t)\leq \kappa\cdot C_t$.

Applying Cauchy-Schwarz we get:\\
$\kappa \sum_{t=0}^{T} C_t \leq 2R||\w^*||\sqrt{T}$

{\bf Cost-Sensitive Perceptron.}
We have the upper bound $||\w_{T+1}||^2 \leq 4R^2\sum_{t=0}^{T}C_t^2$
and a lower bound on $\w_{T+1}\T\w^*$\\
$ = \w_T\T\w^* +C_T \D_T \T \w^* \geq  \w_T\T\w^* + \kappa C_T^2
\geq \kappa \sum_{t=0}^{T} C_t^2 $

Applying Cauchy-Schwarz we get:\\
$\kappa \sum_{t=0}^{T} C_t^2 \leq 2R||\w^*||\sqrt{ \sum_{t=0}^{T} C_t^2 }$\\
$\sqrt{ \sum_{t=0}^{T} C_t^2 } \leq \frac{2R||\w^*||}{\kappa}$ $\Rightarrow$
$\sum_{t=0}^{T} C_t^2 \leq \frac{4R^2||\w^*||^2}{\kappa^2}$

{\bf PA.} We will use the shorthand $S = \frac{M - \w_T\T\D_T}{\D_T\T\D_T}$.
We begin with proving an upper bound on $\D_T\T\D_T$. We know that $\D_T\T\w^* \geq \kappa C_T$. From the Cauchy-Schwarz inequality we then have that $\kappa C_T \leq ||w^*|| \cdot ||\D_T||$. This implies that $\D_T\T\D_T = ||\D_T||^2 \geq \frac{\kappa^2 C_T^2}{ ||w^*|| ^2}$. Since we only consider those examples where the cost is non-zero, this means that $\D_T\T\D_T \geq \frac{\kappa^2}{ ||w^*|| ^2}$.

An upper bound on  $||\w_{T+1}||^2$\\
$ =  ||\w_T||^2 + 2 S\w_T\T \D_T + S^2  \D_T\T \D_T\\$
$=   ||\w_T||^2 + S (2\w_T\T\D_T + (M - \w_T\T\D_T))$
$=  ||\w_T||^2 + S (M + \w_T\T\D_T) $
$=  ||\w_T||^2 + \frac{M^2 -  (\w_T\T\D_T)^2}{\D_T\T\D_T}$
$\leq  ||\w_T||^2 + \frac{M^2}{\D_T\T\D_T} \leq \frac{||\w^*||^2T M^2}{\kappa^2}$

Now a lower bound on $\w_{T+1}\T\w^*$\\
$= \w_T\T\w^* +S \D_T \T \w^*$
\begin{eqnarray*}
&&\geq  \w_T\T\w^* + S \kappa C_T = \w_T\T\w^* + \frac{M - \w_T\T\D_T}{\D_T\T\D_T} \kappa C_T\\
&&\geq \w_T\T\w^*+ \frac{M \kappa C_T}{\D_T\T\D_T} \geq \frac{M\kappa}{4R^2}  \sum_{t=0}^{T} C_t
\end{eqnarray*}

Applying Cauchy-Schwarz we get\\
$M\kappa/(4R^2) \sum_{t=0}^{T} C_t \leq ||\w^*||^2M\sqrt{T}/\kappa$

{\bf Cost-Sensitive PA.} Here we will use the shorthand $S' = (C_T - \w_T\T\D_T)/(\D_T\T\D_T)$, and the bound $\D_T\T\D_T \geq \frac{\kappa^2 C_T^2}{ ||w^*|| ^2}$.

An upper bound on $||\w_{T+1}||^2$\\
$ =  ||\w_T||^2 + 2 S'\w_T\T \D_T + S'^2  \D_T\T \D_T$
$=   ||\w_T||^2 + S' (2\w_T\T\D_T + (C_T - \w_T\T\D_T))$
$=  ||\w_T||^2 + S' (C_tT+ \w_T\T\D_T)$
$=  ||\w_T||^2 + \frac{C_T^2-  (\w_T\T\D_T)^2}{\D_T\T\D_T}$
$\leq  ||\w_T||^2 + \frac{C_T^2}{\D_T\T\D_T} \leq \frac{||\w^*||^2T}{\kappa^2}$

Now a lower bound on $\w_{T+1}\T\w^*$\\
$ = \w_T\T\w^* +S' \D_T \T \w^*$\\
$\geq  \w_T\T\w^* + S' \kappa C_T = \w_T\T\w^* + \frac{C_T - \w_T\T\D_T}{\D_T\T\D_T} \kappa C_T$\\
$\geq \w_T\T\w^*+ \frac{ \kappa C_T^2}{\D_T\T\D_T} \geq \w_T\T\w^*+ \frac{ \kappa C_T^2}{4R^2} \geq \frac{\kappa}{4R^2}  \sum_{t=0}^{T} C_t^2$\\
Using this and applying Cauchy-Schwarz we get: $\frac{\kappa}{4R^2} \sum_{t=0}^{T} C_t^2 \leq \frac{||\w^*||^2\sqrt{T}}{\kappa}$

\end{proof}

\pagebreak

\section*{Appendix B: Bounds for Noisy Data}

In this appendix, we present the analysis of the average cost bounds for the PER and CSPER algorithms in the case where there is noise in the expert's assessment, i.e. the expert might mis-estimate the relative utilities of the original solution and a suggested improvement by a term $\xi_t$ so that Assumption~3 is relaxed, and we only have the assumption that for each example $x_t$,  $U(\langle x_t,y' \rangle) - U(\langle x_t, y \rangle) \geq \kappa C_t - \xi_t$. In this case, for the simple perceptron algorithm we get the bound $\frac{1}{T} \sum_{t=0}^{T} C_t \leq \frac{2R||\w^*||}{\kappa \sqrt{T}} + \frac{1}{\kappa T}\sum_{t=0}^{T} \xi_t$, which is comparable to the bound presented in \cite{DBLP:journals/corr/abs-1205-4213}. For the CSPER algorithm, we obtain the bound $ \frac{1}{T} \sum_{t=0}^{T} C_t \leq \frac{1}{T} \sum_{t=0}^{T} C_t^2$\\ $ \leq\frac{4R^2||\w^*||^2}{\kappa^2T}  + \frac{4R||\w^*||}{\kappa^2T}  \sqrt{\sum_{t=0}^{T}  \xi_t C_t }  + \frac{1}{\kappa^2T} \sum_{t=0}^{T}  \xi_t C_t $.

It is not clear whether similar bounds hold for the PA-based algorithm, since they are susceptible to noise.

\setcounter{trm}{1}
\begin{trm}
\label{theorem}
Given a linear expert utility function with weights $\w^*$, under Assumptions 1-3 we get the following average cost bounds:

PER:$ \frac{1}{T} \sum_{t=0}^{T} C_t \leq \frac{2R||\w^*||}{\kappa \sqrt{T}} + \frac{1}{\kappa T}\sum_{t=0}^{T} \xi_t$

CSPER $ \frac{1}{T} \sum_{t=0}^{T} C_t \leq \frac{1}{T} \sum_{t=0}^{T} C_t^2$\\ $ \leq\frac{4R^2||\w^*||^2}{\kappa^2T}  + \frac{4R||\w^*||}{\kappa^2T}  \sqrt{\sum_{t=0}^{T}  \xi_t C_t }  + \frac{1}{\kappa^2T} \sum_{t=0}^{T}  \xi_t C_t $

\end{trm}

\begin{proof}
{\bf PER}\\
For the Perceptron algorithm, we have the following bound on $||\w_{T+1}||^2$:
\begin{eqnarray*}
||\w_{T+1}||^2 &=& ||\w_T||^2 + 2\w_T\T \D_T + \D_T\T \D_T\\
&\leq&  ||\w_T||^2 +  \D_T\T \D_T\\
&\leq & ||\w_T||^2 + (2R)^2\\
& \leq& 4R^2 T
\end{eqnarray*}

For the lower bound on $\w_{T+1}\T\w^*$ we get:
$$
\w_{T+1}\T\w^* = \w_T\T\w^* + \D_T \T \w^* \geq  \w_T\T\w^* + \kappa C_T  - \xi_T$$
$$ \geq \kappa \sum_{t=0}^{T} C_t - \sum_{t=0}^{T} \xi_t
$$

Combining these and applying Cauchy-Schwarz we get: 
$$\kappa \sum_{t=0}^{T} C_t \leq 2R||\w^*||\sqrt{T} + \sum_{t=0}^{T} \xi_t$$

$$\frac{1}{T} \sum_{t=0}^{T} C_t \leq \frac{2R||\w^*||}{\kappa\sqrt{T}} + \frac{1}{\kappa T}\sum_{t=0}^{T} \xi_t$$

{\bf CSPER}\\

For the CSPER algorithm, we have the following bound on $||\w_{T+1}||^2$:
\begin{eqnarray*}
||\w_{T+1}||^2 &=& ||\w_T||^2 + 2C_T\w_T\T \D_T + C_T^2 \D_T\T \D_T\\
&\leq&  ||\w_T||^2 + C_T^2 \D_T\T \D_T\\
&\leq&  ||\w_T||^2 + (2R)^2C_T^2\\
& \leq& 4R^2 \sum_{t=0}^{T} C_t^2
\end{eqnarray*}

and the following bound on $\w_{T+1}\T\w^*$:
\begin{eqnarray*}
\w_{T+1}\T\w^* &=& \w_T\T\w^* +C_T \D_T \T \w^*\\
& \geq&  \w_T\T\w^* + \kappa C_T^2 - \xi_T C_T\\
&\geq& \kappa \sum_{t=0}^{T} C_t^2 - \sum_{t=0}^{T}  \xi_t C_t
\end{eqnarray*}

Combining these, using Cauchy-Schwarz, we get:
$$ \kappa \sum_{t=0}^{T} C_t^2 - \sum_{t=0}^{T}  \xi_t C_t \leq 2R ||\w^*|| \sqrt{\sum_{t=0}^{T} C_t^2}$$

Using the shorthand notation $\theta =   \sqrt{\sum_{t=0}^{T} C_t^2}$, we can rewrite this as:
$$ \kappa \theta^2 -2R ||\w^*|| \theta -  \sum_{t=0}^{T}  \xi_t C_t \leq 0$$
This is a quadratic inequality in terms of $\theta$. The upper bound for $\theta$ we get from this is:
$$\theta \leq \frac{R||\w^*|| + \sqrt{R^2 ||\w^*||^2 +  \sum_{t=0}^{T}  \xi_t C_t }}{\kappa}$$

For simplicity, we can safely rewrite this (using $\sqrt{A^2 + B^2} \leq \sqrt{A^2} + \sqrt{B^2}$) as:
$$ \theta \leq \frac{2R||\w^*|| + \sqrt{\sum_{t=0}^{T}  \xi_t C_t }}{\kappa}$$

Since $\sum_{t=0}^{T} C_t^2 = \theta^2$, we get:
$$\sum_{t=0}^{T} C_t^2 \leq \frac{4R^2||\w^*||^2}{\kappa^2}  + \frac{4R||\w^*||}{\kappa^2}  \sqrt{\sum_{t=0}^{T}  \xi_t C_t }  + \frac{1}{\kappa^2} \sum_{t=0}^{T}  \xi_t C_t $$

\end{proof}

\pagebreak

\section*{Acknowledgements}
The authors acknowledge support of the ONR ATL program N00014-11-1-0105.

\bibliographystyle{aaai}

\bibliography{coactive}

\end{document}